\documentclass{gSCS2e}%

\usepackage{epstopdf}
\usepackage{subfigure}

\theoremstyle{plain}
\newtheorem{theorem}{Theorem}[section]
\newtheorem{corollary}[theorem]{Corollary}
\newtheorem{lemma}[theorem]{Lemma}
\newtheorem{proposition}[theorem]{Proposition}

\theoremstyle{definition}

\theoremstyle{remark}

\newcommand{\elevel}[2]{{\mathscr{E}_{#1}\inv\!\left(#2\right)}}
\usepackage{graphicx}
\usepackage{latexsym,amsthm,amsmath,amssymb,amsfonts,mathrsfs}
 \usepackage{paralist}
\newcommand{\dd}[1]{\mathrm{d}#1}
\newcommand{\RR}{{\mathbb{R}}}
\newcommand{\RRplus}[1]{{\mathbb{R}^{#1}_{\scriptscriptstyle{\geq 0}}}}
\newcommand{\set}[2]{\left\{#1\,\left|\,#2\right.\right\}}
\newcommand{\card}[1]{{\left|#1\right|}}
\newcommand{\minus}{{\smallsetminus}}
\newcommand{\inv}{{^{-1}}}
\renewcommand{\to}{\longrightarrow}
\newcommand{\THEN}{{\;\Longrightarrow\;}}
\newcommand{\IFF}{{\;\Longleftrightarrow\;}}
\newcommand{\weight}[1]{{\mathfrak{Wgt}_{#1}}}
\newcommand{\energy}[2]{{\mathscr{E}_{#1}\!\left(#2\right)}}
\newcommand{\metr}[1]{{\mathfrak{Met}_{#1}}}
\newcommand{\ultra}[1]{{\mathfrak{Ult}_{#1}}}
\newcommand{\parts}[1]{{\mathscr{P}\!\left(#1\right)}}
\newcommand{\slhc}[1]{{\mathfrak{sl}\!\left(#1\right)}}
\newcommand{\slhcinv}[1]{{\mathfrak{sl}^{{-1}}\!\left(#1\right)}}
\newcommand{\MST}[1]{{\mathtt{MST}\!\left(#1\right)}}
\newcommand{\spt}[1]{{\mathtt{Trees}\!\left(#1\right)}}
\newcommand{\wspt}[1]{{\overline{\spt{#1}}}}
\newcommand{\treemap}[1]{{\alpha\!\left(#1\right)}}
\newcommand{\total}[2]{{\left\Vert #1\right\Vert}_{#2}}

\begin{document}



\title{Statistical Properties of the Single Linkage Hierarchical Clustering Estimator}

\author{
\name{Dekang Zhu\textsuperscript{a}$^{\ast}$\thanks{$^\ast$Corresponding author. Email: dekang.zhu@foxmail.com};
 Dan P. Guralnik\textsuperscript{b};  Xuezhi Wang\textsuperscript{c}; Xiang Li\textsuperscript{a}; Bill Moran\textsuperscript{c}}
\affil{\textsuperscript{a}School of Electronic Science \& Engineering, National University of Defense Technology, Changsha, China;
\textsuperscript{b}Electrical \& Systems Engineering, University of Pennsylvania, Philadelphia, United States;
\textsuperscript{c}Electrical \& Computer Engineering, RMIT University, Melbourne, Australia.
}
}

\maketitle
\begin{abstract}
Distance-based hierarchical clustering (HC) methods are widely used in unsupervised data analysis but few authors take account of uncertainty in the distance data. We incorporate a statistical model of the uncertainty through corruption or noise in the pairwise distances and investigate the problem of estimating the HC as unknown parameters from measurements. Specifically, we focus on single linkage hierarchical clustering (SLHC) and study its geometry. We prove that under fairly reasonable conditions on the probability distribution governing measurements, SLHC is equivalent to maximum partial profile likelihood estimation (MPPLE) with some of the information contained in the data ignored. At the same time, we show that direct evaluation of SLHC on maximum likelihood estimation (MLE) of pairwise distances yields a consistent estimator. Consequently, a full MLE is expected to perform better than SLHC in getting the correct HC results for the ground truth metric.
\end{abstract}


\begin{classcode}
62H30; 62H12;68Q87;94A15
\end{classcode}

\section{Introduction}\label{section:introduction} Hierarchical clustering (HC) is widely used in a range of applications, especially for situations where the hierarchy is physically meaningful, such as the study of the genetic history of biological populations (evolutionary trees)~\cite{khanafiah2006visualizing,blanchette2012inference} and microarray expression data which contributes to pharmacology, toxicogenomics and disease-subclass determination~\cite{butte2002use,levenstien2003statistical}.
Phylogeny of languages~\cite{mahata2006hierarchical} and authorship attribution~\cite{Segarra_Ribeiro-authorship_attribution} are performed using HC methods. In addition, HC is used where ``flat'' clustering is required but the scale of the clustering remains to be determined.
In wireless sensor network (WSN) routing protocol designs, HC methods have been introduced for grouping sensor nodes into clusters to achieve network scalability, energy efficiency and longer lifetimes~\cite{lung2007applying,lung2010using}.
In the context of WSNs, increasingly being deployed in Internet of Things (IoT) scenarios, energy efficiency and better organization are achieved by using HC methods to form clusters of sensors, enabling the network to operate in a grouped way.
Of particular importance, applications have been proposed in cyber-security, e.g. detection of distributed denial of service (DDoS)
attacks~\cite{karami:2013,conf_sbp_DuY11}, detection of worms and viruses~\cite{Wang_Miller_Kesidis-cluster_hierarchy_for_anomaly_detection} and ``poisoning'' methods have been proposed to disrupt HC~\cite{biggio2014poisoning}.
HC methods are applied in many other domains, including documents classification~\cite{Steinbach2000}, image segmentation~\cite{martinez2006unsupervised}, network topology identification~\cite{Castro:2003fk}, and so on.
Many clustering algorithms begin with a population being assigned some
measure of difference (weight) between members of that population.
Typically, but not always, the pre-processing of the data involves modelling the differences between members of the population by a metric rather than a general weight. That is, the differences are assumed to be symmetric and, more significantly, to satisfy the triangle inequality.
As a result we will restrict our attention to clustering
methods where the weights are actually metrics --- \emph{distance-based clustering}.

There are many approaches to clustering, and HC, of a finite population with a metric to quantify differences (finite metric space); for instance, single/complete linkage clustering and unweighted pair group method with arithmetic mean (UPGMA).
``Flat'' clustering methods are often unable to explicate the finer structure of clusters, and the most significant disadvantage for many of them is the dependence on the number of clusters, which is almost never more than a guess.
Most flat clustering methods are optimizations with respect to some objective functions over all possible clusterings, and it is infeasible for a global search. Therefore, for implementation, they start from an initial random partition and refine it iteratively to get a ``local'' optimum. Inevitably, finding a good starting point is a key issue.
On the contrary, HC methods return a hierarchy of partitions unveiling structure in the data and do not usually require a pre-specified number of clusters. In fact, HC methods are helpful in identifying the correct number of clusters.
Furthermore, HC methods are deterministic. Disadvantages of flat clustering and advantages of HC are discussed in~\cite{manning2008introduction}.

Meanwhile, attempts to justify clustering from an axiomatic standpoint have been made, notably one by Kleinberg~\cite{Kleinberg-impossibility} which resulted in an ``impossibility theorem'' stating a seemingly minimalistic and natural set of requirements of a ``good'' distance-based clustering algorithm that cannot be satisfied simultaneously.
However, as argued in~\cite{Carlsson2010}, HC overcomes this non-existence problem and gives an analogous theorem in which one obtains an existence and uniqueness theorem instead. Moreover, this perspective affords a rich mathematical theory with deep roots in geometry~\cite{Isbell-injective_envelope} and topology~\cite{Carlsson_Memoli-classifying_clustering_schemes}.
In addition to being the only known tractable, unsupervised, well-defined HC method, single linkage hierarchical clustering (SLHC) has particular mathematical properties that make it an attractive option. Only SLHC is stable under small perturbations of the weights, and it is the only one satisfying the following consistency/convergence property~\cite{Carlsson2010}: if the number of i.i.d. sample points goes to infinity, the result of applying SLHC to a data set converges almost surely in the Gromov-Hausdorff sense to an ultra-metric space recovering the multi-scale structure of the probability distribution support.

Despite the above suite of attractive characterizations, a focus on the so-called ``chaining'' phenomenon in SLHC contributes to a growing consensus that SLHC is largely unsuitable for ``real world'' applications~\cite{Jain:1999:DCR:331499.331504}.
It is important to observe, however, that the chaining phenomenon may not be as pertinent a feature of the resulting HC if one allows small perturbations of the weights~\cite{Gama_Segarra_Ribeiro-dithering}.
Indeed, we argue that it is often the case that the measurement/assignment of difference carries some level of uncertainty. This uncertainty can be an inherent ``noise'' in the measurement process but in many applications might just be a technique for modelling the unknowns in the weight assignment process. A typical application might be an attempt to cluster on the basis of common features characterised by a real parameter (temperature, length, etc), where there is uncertainty in the actual value of the parameter.
This work adopts and explores the point of view that the uncertainties inherent in the modelling process require us to view the weights it generates as measurements of an unknown ``true'' metric describing a correct model of the observed population and the differences between its members. The aim here is to estimate (in the statistical sense) the most likely hierarchical clustering of the data associated with that measurement.

Conventional approaches to statistical estimation of partitions and hierarchies view the objects to be clustered as random samples of certain distributions over a prescribed geometry (e.g. Gaussian mixture model estimation using expectation-maximization in Euclidean spaces), and clusters can then easily be described in terms of their most likely origin. Thus, these are really {\it distribution-based} clustering methods~ ---~ not distance-based ones. Our approach, proposed here for the first time, directly attributes uncertainty to the process of obtaining values for the pairwise distances rather than distort the data by mapping it into one's ``favorite space''. To the best of our knowledge, very little work has been done in this vein. Of note is~\cite{Castro04likelihoodbased}, where similar ideas have been applied to the estimation of spanning trees in a communication network. However, pairwise dissimilarity measures characterizing the latent hierarchy are already given in their setting, and {\it additional} data is then applied to infer the correct tree. Thus, their problem may be seen as complementary to ours (see Section~\ref{subsubsec:Preimage_of_u} below).

Our approach to the estimation of the single linkage hierarchy corresponding to a ground truth metric $\theta$ treats $\theta$ as a nuisance parameter to be eliminated. In situations of this kind, it is recommended~\cite{berger1999integrated} to maximize an integrated likelihood function derived from an appropriate so-called ``conditional prior''.
As is often the case, the full maximum likelihood estimator (MLE) is intractable and we have to resort to weaker versions of the optimization.
In particular, we consider certain profile and partial likelihoods. We recall that profile likelihood eliminates the nuisance parameter through maximization rather than integration~\cite{berger1999integrated}, the former being easier to implement. Partial likelihood is based on the idea that the problem parameters may be partitioned into ``blocks'', one of which strongly depends on the parameter of interest, so that the likelihood for that block (given the parameter) is easier to compute~\cite{cox1975partial}.

This paper is organized as follows. Section~\ref{section:prelim} gives a review of preliminary notions useful in establishing the relationship between SLHC and the geometry of the metric cone in terms of spanning trees. This facilitates the discussion of statistical estimation of SLHC in the main section, Section~\ref{section:estimation}. In Section~\ref{section:consistency} we prove the consistency of SLHC as an estimator. We conclude with simulations and a brief discussion in Section~\ref{section:conclusion}.\\

\section{Preliminaries}\label{section:prelim}
We restrict our attention to {\em distance-based} clustering methods, for which it is assumed that a data set $O$ first undergoes initial processing to produce a \emph{weight} (see below), which is then used as input to a clustering map.
A distinction is made in the literature between \emph{flat} clustering methods and \emph{hierarchical} ones: a flat clustering method generates a partition of $O$ from a weight $d$, while a hierarchical clustering method produces a hierarchy of partitions, also known as a dendrogram (see below).

\subsection{Hierarchies and Dendrograms}\label{section:notions}
\subsubsection{Weights and Metrics}\label{section:metrics}
A {\em weight}\footnote{Often called a {\em dissimilarity} in this context} on $O$ is a symmetric function $d\colon O\times O\to\RRplus{}$ satisfying $d(x,x)=0$ for all $x\in O$, whose values $d(x,y)$, $x,y\in O$ represent a ``degree of dissimilarity'' between data entries. We will say that a weight $d$ is {\it strict}, if $d(x,y)>0$ whenever $x\neq y$.
Hereafter we will use the notation $d_e=d_{xy}:=d(x,y)$, with the pair $e=xy$ representing an undirected edge of the complete graph $K_O$ with vertex set $O$ whenever $x\neq y$. The set of edges of $K_O$ will be denoted by $\binom{O}{2}$.

Metrics on $O$ traditionally form a preferred class of weights for clustering purposes. A weight $d$ on $O$ is called a \emph{(semi/pseudo)-metric} if it satisfies the \emph{triangle inequality}, $d_{xz}\leq d_{xy}+d_{yz}$. An \emph{ultra-metric} is a weight satisfying an even stronger condition, the \emph{ultra-metric inequality}, $d_{xz}\leq\max\left\{d_{xy},d_{xz}\right\}$. The spaces of weights, metrics and ultra-metrics on $O$ will be denoted by $\weight{O}, \metr{O}$ and $\ultra{O}$, respectively.

\subsubsection{Partitions}\label{section:partitions}
Recall that a collection $R$ of pairwise-disjoint subsets of $O$ whose union is $O$ is called a {\em partition} of $O$. The elements of a partition will be referred to as its {\em clusters}. We denote the set of partitions of $O$ by $\parts{O}$. Thus, formally, a distance-based flat clustering method on $O$ is merely a function $\metr{O}\to\parts{O}$.

For any $x\in O$ and any partition $R\in\parts{O}$, we will denote the cluster of $R$ containing $x$ by $R_x$. Recall that a partition $R$ is said to \emph{refine} a partition $R'$, if every cluster of $R$ is contained in a cluster of $R'$; we denote this by $R\succeq R'$. A \emph{hierarchy} is a \emph{chain of partitions}: a collection $\mathcal{C}$ of partitions where every $R,R'\in\mathcal{C}$ satisfy either $R\succeq R'$ or $R'\succeq R$. The refinement relation has been shown by Kleinberg~\cite{Kleinberg-impossibility} to be fundamental in any principled discussion of distance-based clustering maps, because of its role as an obstruction to a variety of intuitive consistency requirements commonly seen as desirable for a clustering method.

\subsubsection{Dendrograms}\label{section:dendrograms} Following~\cite{Carlsson2010}, we describe a dendrogram as a pair $(O,\beta)$, where $\beta:[0,\infty)\to\parts{O}$ is a map satisfying the following: (1) there exists $r_0$ s.t. $\beta(r)=\{O\}$ for all $r\geqslant r_0$, (2) if $r_1\leqslant r_2$ then partition $\beta(r_1)$ refines partition $\beta(r_2)$, and (3) for all $r$ there exists $\epsilon>0$ s.t. $\beta(t) = \beta(r)$ for $t\in[r,r+\epsilon]$. Carlsson and M\'emoli have shown that expanding the domain of allowed classifiers from $\parts{O}$ to dendrograms over $O$ resolves the consistency problems indicated in Kleinberg's work in~\cite{Kleinberg-impossibility}.
It will be useful to separate the metric information in a dendrogram (the grading by resolution) from the combinatorial information it conveys: a dendrogram may be uniquely represented by a pair $(\tau,\bm{a})$, where $\tau$ denotes the {\it structure of $u$}~ ---~ the chain of partitions defined by $u$ (with the resolutions forgotten), ordered by refinement; and $\bm{a}$ is the {\it height vector}, see Figure~\ref{fig:dendrograms}, whose coordinates, in order, indicate the minimum resolution at which each partition in the structure occurs in the dendrogram.

Ultra-metrics provide a convenient tool for encoding dendrograms~---~see~\cite{Jardine71,Carlsson2010} and Figure~\ref{fig:dendrograms}: any dendrogram $\beta$ gives rise to an ultra-metric $u=u(\beta)$, and conversely, an ultra-metric $u$ encodes a dendrogram $\beta_u$ via:
\begin{equation}
	u(\beta)_{xy}:=\inf\set{r>0}{\beta(r)_x=\beta(r)_y}\,,\quad
	\beta(r)_x:=\set{y\in O}{u_{xy}\leq r}\,.
\end{equation}
This enables a formulation of the study of \emph{distance-based hierarchical clustering} as the study of `coherent' families of maps $\left\{cl_{O}:\metr{O}\to\ultra{O}\right\}_{O\neq\varnothing}$.\footnote{Where coherence is to be understood in suitable categorical terms, as suggested by Carlsson and M\'emoli~\cite{Carlsson_Memoli-categorical,Carlsson_Memoli-classifying_clustering_schemes}.}
\begin{figure}[t]
	\begin{center}
		\includegraphics[width=.5\columnwidth]{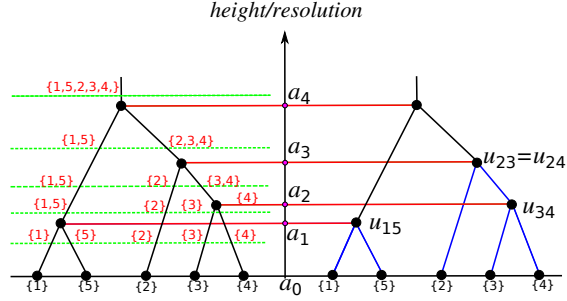}
		\caption{A rooted tree with labelled leaves as a dendrogram (left) and as an ultra-metric on $O=\{1,\ldots,5\}$ (right).}\label{fig:dendrograms}
	\end{center}
\end{figure}

\subsection{Single Linkage Hierarchical Clustering (SLHC)}\label{section:slhc} 
\subsubsection{Construction of SLHC}\label{section:slhc construction} Recalling~\cite{Carlsson2010}, one defines single-linkage hierarchical clustering of a weighted space $(O,d)$ to be the dendrogram $\beta_d$ for which $\beta_d(r)_x=\beta_d(r)_y$ if and only if $O$ contains points $x_0,\ldots,x_m\in O$ where $m\in\mathbb{N}$, $x_0=x$, $x_m=y$ and $d(x_{i-1},x_i)\leq r$ for $i\in\{0,\ldots,m\}$. Following~\cite{Gower69}, SLHC is often implemented by constructing a minimum spanning tree (MST) in $(O,d)$: the partition $\beta_d(r)$ is obtained from \emph{any} MST $T$ of $(O,d)$ as the set of connected components obtained from $T$ by removing all edges of $T$ of length exceeding $r$. 
The corresponding ultra-metric, which we denote by $\slhc{d}$, has $\slhc{d}_{xy}$ equal to the maximum $d$-length among all edges in $T$ separating $x$ from $y$. It is well known that:
\begin{equation}\label{eq:slhc as sup}
	\slhc{d}_{xy}=\sup\set{u_{xy}}{u\in\ultra{O}\,,\; u\leq d}\,.
\end{equation}
From this presentation, some properties emerge which characterize single linkage among all operators $E:\metr{O}\to\metr{O}$:
\begin{proposition}\label{prop:slhc characterization} There exists one and only one operator $E:\metr{O}\to\metr{O}$ with $E(\metr{O})=\ultra{O}$ such that
	\begin{equation}
	\mathrm{(a)}\;\; EEd=Ed\,,\qquad
	\mathrm{(b)}\;\; Ed\leq d\,,\qquad
	\mathrm{(c)}\;\; d\leq d'\THEN Ed\leq Ed'\,,
	\end{equation}
	for all $d,d'\in\metr{O}$, and this operator is SLHC over $O$.
\end{proposition}
This simple characterization of the single linkage operator, encompassing very general notions of (a) stability; (b) information bottlenecking; and (c) consistency, appears to have gone unnoticed up till now (compare with the categorical characterizations by Carlsson and M\'emoli~\cite{Carlsson_Memoli-categorical}), and serve as our motivation for studying SLHC a target for statistical estimation in the present context.

Proofs of this Proposition and of the remaining technical results regarding the geometry of SLHC presented in this section are given in Appendix~\ref{app:Proofs_of_materials_in_preliminary}. To the best of our knowledge, these results are new.

\subsubsection{Weighted Spanning Trees}\label{section:weighted trees} In view of the preceding paragraph, it will be convenient to introduce the space of weighted spanning trees over $O$. First, we denote the set of spanning trees in $K_O$ by $\spt{O}$, while identifying each spanning tree $T$ with its edge set. For a weight $d\in\weight{O}$, the subset of its MSTs will be denoted by $\MST{d}$.

Define the space $\wspt{O}$ of {\it weighted spanning trees} to be the space of all pairs $(T,w)$~ ---~ henceforth denoted $T^w$ in this context~ ---~ with $T\in\spt{O}$ and $w:T\to\RRplus{}$ being referred to as a {\it weight on $T$}. The {\it total weight} of $T^w$ is defined to be $\total{T}{w}:=\sum_{e\in T}w_e$. Of course, a weight $d$ on $O$ naturally restricts to a weight $w$ on $T$ through $w_e:=d_e$ when $e\in T$; we denote the corresponding weighted spanning tree by $T^d$, by abuse of notation.

The obvious identification(s) of the set of all weights on a fixed tree $T$ with $\RRplus{n-1}$ yields a topology on $\wspt{O}$ with $\card{\spt{O}}=n^{n-2}$ clopen\footnote{That is, both closed and open~\cite{Munkres-topology}.} connected components, one for each $T\in\spt{O}$, each homeomorphic to $\RRplus{n-1}$. 

For a given spanning tree $T$, note that every $xy\in\binom{O}{2}$ determines a unique edge path in $T$ joining $x$ with $y$. We denote this path by $p(T)_{xy}$. We then observe that the construction of the preceding paragraph makes use of the map $\alpha:\wspt{O}\to\ultra{O}$ sending every $T^w\in\wspt{O}$ to the ultra-metric defined by:
\begin{equation}\label{eqn:ultra-metric from a tree}
	\treemap{T^w}_{xy}:=\max\left\{w_e\,\left| e\in p(T)_{xy} \right.\right\}.
\end{equation}
The continuity of $\alpha$ is fairly straightforward. Curiously, it is central to establishing the consistency of SLHC as an estimator (Theorem~\ref{thm:consistency_of_SLHC}), so we include a proof of this fact in this paper for the sake of completeness (Lemma~\ref{lemma:continuity of alpha}).

The following elementary result has surprisingly significant impact. We have not been able to locate it or its corollaries in the literature.
\begin{lemma}\label{lemma:mst comparison} Let $v,w$ be real-valued weights on the edges of the complete graph $K_O$. Suppose that $v_e<v_f$ implies $w_e<w_f$ for all $e,f\in {O\choose 2}$. Then $\MST{v}\subseteq\MST{w}$.
\end{lemma}

\begin{corollary}\label{cor:mst invariance under increasing functions} Let $w$ be a real-valued weight on the edges of the complete graph $K_O$. If $g:\RR\to\RR$ is a strictly increasing function, then $T$ is an MST of $w$ iff it is an MST of $g\circ w$.
\end{corollary}

\subsubsection{The geometry of fibers of SLHC}
\label{subsubsec:Preimage_of_u}
Estimating the parameter $\bm{u}=\slhc{\bm{\theta}}$ from a measurement $\bm{x}$ of a ``ground truth'' metric $\bm{\theta}\in\metr{O}$ while treating the remaining information regarding $\bm{\theta}$ as a nuisance parameter requires detailed understanding of the point pre-images of the map $\slhc{\cdot}$, which we develop below (we shall, from now on, use bold symbols to denote the weight/metric/ultra-metric as multi-dimentional parameters in the context of the estimation problem discussed in this paper).

For any weight $w\in\weight{O}$ we will say that $w$ is {\it generic}, if no two values of $w$ coincide, that is: if $w_e\neq w_f$ for all $e,f\in{O\choose 2}$ with $e\neq f$. It is a textbook result that a generic metric has exactly one minimum spanning tree~\cite{West-textbook}. This may be restated geometrically as follows:
\begin{lemma}\label{lemma:spanning tree decomposition} For each $T\in\spt{O}$, define $C(T)$ to be the set of all $d\in\weight{O}$ satisfying $T\in\MST{d}$. Then $\weight{O}$ is the union of the pairwise interiorly-disjoint domains $\{C(T)\}_{T\in\spt{O}}$.\qedhere
\end{lemma}

\begin{lemma}\label{lemma:cone of a tree} For $T\in\spt{O}$, the set $C(T)$ is a closed pointed convex cone.
\end{lemma}

\begin{lemma}\label{lemma:sl inverse images} For every $u\in\ultra{O}$ we have:
	\begin{equation}\label{preimage decomposition}
	\slhcinv{u}=\bigcup_{T\in\MST{u}}W(T,u)\,,\quad
	W(T,u):=\set{w\in\weight{O}}{T^w=T^u\,,\;w\geq u}.
	\end{equation}
	Each $W(T,u)$ is a finite-sided convex polytope and no two of these polytopes have an interior point in common.
\end{lemma}
A critical distinction between the case of general weights/dissimilarities and the metric case is that, while the $W(T,u)$ are unbounded polytopes, the polytopes
\begin{equation}\label{eq:C(T,u)}
C(T,u):=W(T,u)\cap\metr{O}
\end{equation}
are compact, due to the characterization
\begin{equation}\label{eq:inequalities for fiber}
C(T,u)=\set{d\in\metr{O}}{u\leq d\leq\omega(T^u)}\,,
\end{equation}
where $\omega(T^u)$ is the {\it tree metric on $O$ induced from the weights given by $u$}:
\begin{equation}
\omega(T^u)_{xy}=\sum_{e\in p(T)_{xy}}d_e.
\end{equation}
The triangle inequality forces $d\leq\omega(T^u)$ on any metric $d$ satisfying $T^d=T^u$. Note that $\omega(T^u)$ coincides with $u=\alpha(T^u)$ on edges of $T$, as $T\in\MST{u}$.

The main implications of the resulting decomposition
\begin{equation}\label{eq:slhc fiber}
\slhcinv{u}\cap\metr{O}=\bigcup_{T\in\MST{u}}C(T,u)
\end{equation}
for this work are significant: the single linkage operator $\slhc{\cdot}$ is only smooth when restricted to one of the $C(T)$, hence any explicit integration over $\slhcinv{u}\cap\metr{O}$ needs to be decomposed as a sum of integrals over the relevant $C(T,u)$. Unfortunately, the structure (e.g. set of vertices) of these polytopes is known to be extremely complex, as they share vertices with the set of extreme directions of the metric cone $\metr{O}$, whose enumeration is a long-standing open problem~\cite{Avis,Deza}. This makes obtaining closed-form integration formulae over $C(T,u)$~ ---~ and marginalizing over $C(T,u)$ in particular~ ---~ a very hard combinatorial problem. In addition, the number of MSTs for a given ultra-metric poses an additional hurdle on the way to solving optimization problems over $\slhcinv{u}\cap\metr{O}$.

\section{Estimation of SLHC}\label{section:estimation}
\subsection{Statistical model for hierarchical clustering}
Distances between data points are assumed to be corrupted by noise from a suitable distribution. One na\"ive approach to estimating hierarchical clustering is to estimate the true distances between data points and feed them into the given HC algorithm. This is almost certainly not optimal in the highly nonlinear context of HC methods, none of which~ ---~ at least among those satisfying requirement (a) of Proposition~\ref{prop:slhc characterization}, which is quite natural~ ---~ are one-to-one maps (for example, consider SLHC; $\slhc{\cdot}$ is a piecewise-smooth map whose generic fibers have co-dimension $(n-1)$ in the $\binom{n}{2}$-dimensional space $\metr{O}$, as seen from Equations \eqref{eq:slhc fiber} and \eqref{eq:inequalities for fiber}). Therefore, our statistical model is as follows:
\begin{compactitem}
	\item The ultra-metric $\bm{u}=cl_O(\bm{\theta})$ of a hierarchical clustering map $cl_O$ is a fixed unknown parameter to be estimated from a measurement $\bm{x}$ of a metric $\bm{\theta}$; thus, $\bm{\theta}$ is a nuisance parameter taking its values in the point pre-image $(cl_O)\inv({\bm{u}})$.
	\item The measurement $\bm{x}\in \weight{O}$ only depends on $\bm{\theta}$ through a specific probability distribution which we denote by $\pi(\bm{X}|\bm{\theta})$. 
	\item A reasonable assumption for this noise model is that the measurements $x_e$ of the values $\theta_e$, $e\in\binom{O}{2}$ of $\bm{\theta}$ are sampled independently from the same parametrized distribution $G_\theta(X)$, $\theta\in(0,+\infty)$:
\begin{equation}
	\pi(\bm{X}|\bm{\theta}) = \prod _{e\in{O\choose 2}} G_{\theta_{e}}(X_{e}).
\end{equation}
\end{compactitem}
These assumptions are justified, for example, in the context of WSNs as discussed in Section~\ref{section:introduction}.  The input data for clustering is provided in the form of independently measured distances between the sensors, obtained from Received Signal Strengths (RSSs), parametrized by the ground truth distances and corrupted by receiver noise as well as external effects such as multipath or shadowing. These kinds of distortions are often, and justifiably, modeled as noise~\cite{mao2007wireless,cox1984800,bernhardt1987macroscopic}. 

Construction of the likelihood $p(\bm{x}|\bm{u})$
from $\pi(\bm{x}|\bm{\theta)}$ and $\bm{u} = cl_O(\bm{\theta})$
is actually a problem of eliminating the nuisance parameter
$\bm{\theta}$. 
Hierarchical clustering maps are generally not one-to-one,
and this is the case for  $\slhc{\cdot}$.  Because of the uncertainty of
$\bm{\theta}$ relative  to $\bm{u}$, classical likelihood methods
for eliminating nuisance parameters (such as profile likelihood)
perform worse than the integrated likelihood proposed by Berger~\cite{berger1999integrated}. As pointed out there, going back to
papers of \cite{neyman1948consistent,cruddas1989time},  and as we have observed in our problem, profile
likelihood  can give rise to ``misleading behaviour''. 
Section~2 of the Berger \emph{et.al} paper argues in
favour of  integrated likelihood on the ``grounds of simplicity,
generality, sensitivity, and precision''. 
The integrated likelihood $\mathcal{L}(\bm{u};\bm{x})$ 
is constructed by integrating over $\bm{\theta}$ according to,
in the terms of Berger \emph{et al.}, ``the conditional prior
density'' of $\bm{\theta}$ given $\bm{u}$:
\begin{equation}\label{eq:inner_integral}
p(\bm{x}|{\bm{u}})
	=\int p(\bm{x}|\bm{\theta},\bm{u}) p(\bm{\theta}|{\bm{u}})\dd{\bm{\theta}}
	=\int \pi(\bm{x}|\bm{\theta}) p(\bm{\theta}|\bm{u})\dd{\bm{\theta}}
\end{equation} 
In the current context the support of $G_\theta$ is restrict to $(0,+\infty)$.

\subsection{Specializing to Single Linkage}
In the case of SLHC, our estimation model may be further decomposed through treating the MST of $\bm{\theta}$ (which is also an MST of the ultra-metric $\bm{u}$) as a discrete nuisance parameter associated with $\bm{u}$. Generically speaking, $\bm{u}$ has multiple MSTs, so any integration over the domain $\slhcinv{\bm{u}}\cap \metr{O}$ must decompose according to~\eqref{eq:slhc fiber}. Consequently, the weight (prior) $p(\bm{\theta}|\bm{u})$ decomposes as
\begin{equation} \label{eq:prior_decomposion_according_T}
	p(\bm{\theta}|\bm{u}) = \sum_{T\in \MST{\bm{u}}} p(T|\bm{u}) p(\bm{\theta}|T,\bm{u})\,,
\end{equation}
which itself may be seen as an integrated likelihood with respect to the discrete parameter $T$.

Substituting~\eqref{eq:prior_decomposion_according_T} into~\eqref{eq:inner_integral}, we obtain the likelihood:
\begin{equation}\label{eq:integrated_likelihood_for_u_slhc}
\begin{split}
	p(\bm{x}|{\bm{u}})
		&= \int_{\slhcinv{\bm{u}}\cap \metr{O}} \pi(\bm{x}|\bm{\theta})  \sum_{T\in \MST{\bm{u}}} p(T|\bm{u}) p(\bm{\theta}|T,\bm{u}) \dd{\bm{\theta}} \\
		&=\sum_{T\in\MST{\bm{u}}} p(T|\bm{u})\!\! \int\limits_{C(T,\bm{u})}\! \pi(\bm{x}|\bm{\theta})p(\bm{\theta}|T,\bm{u}) \dd{\bm{\theta}}.
\end{split}
\end{equation}
Figure~\ref{fig:concept_illustration_statistical_model}~is a schematic for the statistical model for estimating SLHC.
\begin{figure}[!t]
	\begin{center}
		\includegraphics[width=.8\columnwidth]{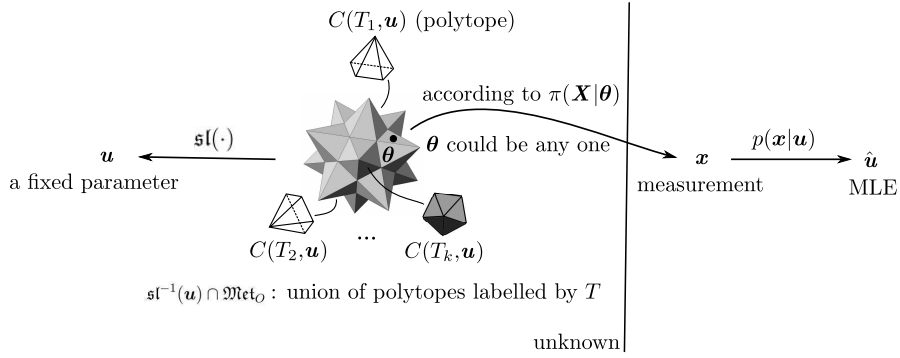}
		\caption{A schematic of the statistical model for estimating SLHC.}\label{fig:illustration_statistical_model}
		\label{fig:concept_illustration_statistical_model}
	\end{center}
\end{figure}

\subsubsection{``Least-Biased'' Distributions and Restricting to $\metr{O}$}\label{section:why uniform on fibers?} Below we argue that, in the absence of assumptions on the latent parameters $\bm{\theta}$, the weight distributions $p(\bm{\theta}|\bm{u})$ ought to be uniform over their supports. Hence:
\begin{equation}
	p(\bm{\theta}|\bm{u})=\left\{\begin{array}{cl}
		1\big{/}\mathrm{Vol}\left (\slhcinv{\bm{u}}\cap \metr{O}\right)
			&\text{ when } \bm{\theta} \in \slhcinv{\bm{u}}\cap \metr{O}\\
		0	&\text{ otherwise }
	\end{array}\right.
\end{equation}
Accordingly, $p(\bm{\theta}|T,\bm{u})$ is uniformly distributed on $C(T,\bm{u})$, and the $p(T|\bm{u})$ should be assigned according to their relative volumes. Unfortunately, $p(T|\bm{u})$ requires the computation of integrals over the polytopes $C(T,\bm{u})$ (see Section~\ref{subsubsec:Preimage_of_u} above). Existing techniques for these computations incur extremely high computational costs~\cite{barvinok1993computing,barvinok2010maximum,lasserre01laplace}.

The above assumption is justified by the following maximum entropy argument, the details of which lie outside the scope of this paper and will be detailed elsewhere. It follows from \eqref{eq:slhc fiber} that the weight $\total{T}{w}$ of $T\in\MST{w}$ is constant for $w\in\slhcinv{{u}}$ for a fixed ${u}\in\ultra{O}$. Moreover Corollary~\ref{cor:mst invariance under increasing functions} guarantees $\MST{g\circ w}=\MST{w}$ for all $w\in\metr{O}$ whenever $g:[0,\infty)\to[0,\infty)$ is {\it any} strictly increasing function, so that the quantity 
\begin{equation}\label{eq:energy functional}
	\energy{g}{w}:=\total{T}{g\circ w}\,,\qquad T\in\MST{w}
\end{equation}
is a reasonable choice of energy on the space $\metr{O}$, viewed as a space of micro-states of the `particles' in $O$. In turn, one can show that $\energy{g}{\cdot}$ gives rise to a piecewise-smooth maximum entropy parametric distribution
\begin{equation}\label{eq:maxent}
	p_g(w)\varpropto\mathtt{exp}\left(-\mu\energy{g}{w}\right)
\end{equation}
on $\metr{O}$, derived from the constraints
\begin{equation}\label{eq:energy constraint}
	\int_{\metr{O}}p_g(w)\dd{w}=1\,,\qquad
	\int_{\metr{O}}\energy{g}{w}p_g(w)\dd{w}=\mu.
\end{equation}
Clearly by \eqref{eq:slhc fiber}, this distribution is constant on $\slhcinv{{u}}\cap \metr{O}$ for all ${u}\in\ultra{O}$. 

Crucially, the above considerations also speak in favor of our self-imposed restriction to ground truth weights satisfying the triangle inequality. Observe that for each $T\in\spt{O}$, the level sets $\elevel{g}{r}$, $r\geq 0$, of the functional $\energy{g}{\cdot}$ only depend on the weights induced on $T$. This implies that, in $\weight{O}$, these level sets are unbounded. Consequently, since $p_g(\cdot)$ is a function of the energy alone, $p_g(\cdot)$ will not be normalizable under these circumstances. Restricting the possible values of $\theta$ to $\metr{O}$ solves this problem because the triangle inequality guarantees the compactness of level sets. Indeed, by Lemma~\ref{lemma:spanning tree decomposition}, any level set $\elevel{g}{r}$, $r\geq 0$, is a finite union, over all spanning trees $T\in\spt{O}$, of the closed subsets whose elements are {\em metrics} $w$ with $\total{T}{g\circ w}=\alpha$ and with $T\in\spt{O}$ fixed; by the triangle inequalities for $w$, no value of $w$ may exceed the total weight of $T$, implying that the level set is a finite union of closed bounded sets.

\subsubsection{Partial Profile Likelihood} 
The likelihood proposed in~\eqref{eq:integrated_likelihood_for_u_slhc} incurs a prohibitive computational cost, if precise computation is required. Leaving computationally feasible approximation methods of the full likelihood to future work, here, instead, we propose eliminating the discrete nuisance parameter $T\in\MST{\bm{u}}$ by forming the following profile likelihood:
\begin{equation}\label{eq:profile_likelihood_of_p_x_u}
	\mathcal{L}_{pr}(\bm{u};\bm{x}) = \underset{T\in \MST{\bm{u}}}{\text{max}}p(\bm{x}|T,{\bm{u}}) =  \underset{T\in \MST{\bm{u}}}{\text{max}} \int\limits_{C(T,\bm{u})}\! \pi(\bm{x}|\bm{\theta})p(\bm{\theta}|T,\bm{u}) \dd{\bm{\theta}}.
\end{equation}
For this notion of likelihood we have the following result:
\begin{proposition}\label{prop:MLE_u_profile_likelihood} Let $\alpha:\wspt{O}\to\ultra{O}$ be the map defined in Section~\ref{section:weighted trees}. Then, the maximum likelihood estimate of $\bm{u}$ derived from the profile likelihood of \eqref{eq:profile_likelihood_of_p_x_u} is given by
	\begin{equation}
	\hat{\bm{u}} = \treemap{
		\arg\underset{T^{\bm{u}} \in \wspt{O}}{\max}\; p(\bm{x}|T,{\bm{u}})
	}.
\end{equation}
\end{proposition}

\begin{proof} 
Equation \eqref{eq:profile_likelihood_of_p_x_u} is the modified likelihood defined in~\cite{Kay:1993:FSS:151045}, and the proposition follows from the invariance property of the maximum likelihood estimator proved there.
\end{proof}

Consider the two complementary projections of $\bm{\theta}\in C(T,{\bm{u}})$ denoted $\bm{\theta}^{\text{on}} = (\theta_e)_{e\in T}$ and $\bm{\theta}^{\text{off}} = (\theta_e)_{e\notin T}$, with the variable $\bm{X}$ split into $\bm{X}^{\text{on}}$ and $\bm{X}^{\text{off}}$ accordingly.

\begin{proposition}\label{prop:partial_likelihood} For any $T\in \MST{\bm{u}}$, $p(\bm{X}|T,{\bm{u}})$ splits as the product of two independent parts: 
\begin{equation} \label{eq:two_parts_likelihood}
p(\bm{X}|T,{\bm{u}}) = p(\bm{X}^{\text{on}}|T,{\bm{u}})p(\bm{X}^{\text{off}}|T,{\bm{u}})\,,
\end{equation} 
where:
\begin{equation} \label{eq:partial_likelihood_in_proof}
\left \{
\begin{aligned}
& p(\bm{X}^{\text{on}}|T,{\bm{u}}) = \prod_{e\in T} G_{u_e}(X_e)\,,\\
& p(\bm{X}^{\text{off}}|T,{\bm{u}}) = \int_{C(T,{\bm{u}})} \prod_{e\notin T}G_{\theta_e}(X_e)p(\bm{\theta}^{\text{off}}|T,{\bm{u}})\dd{\bm{\theta}^{\text{off}}}\,.\\
\end{aligned}
\right.
\end{equation}
\end{proposition}
\begin{proof} See Appendix~\ref{append:proof_of_independence_partial}.
\end{proof}

The product decomposition above raises the question of how much hierarchical information is conveyed {\em only} by the distance measurements along the edges of a chosen tree. One expects SLHC to correspond to the best estimate in this case. Indeed, the main result of this section characterizes SLHC as a maximum {\em partial} profile likelihood estimator (MPPLE) of hierarchical clustering for a significant class of measurement models. We define the partial profile likelihood~\cite{cox1975partial} of $\bm{u}$ given a single measurement $\bm{x}$ to be:
\begin{equation} \label{eq:definition_of_partial_likelihood}
\mathcal{L}_{pp} ({\bm{u}};\bm{x}) =  \underset{T\in \MST{\bm{u}}}{\text{max}} p(\bm{x}^{\text{on}}|T,{\bm{u}}).
\end{equation}
The resulting characterization is as follows:
\begin{theorem}\label{thm:slhc_maxi_PLE}
Given the measurement model $G_\theta$, let $g(x)=G_{\hat{\theta}(x)}(x)$ where $\hat{\theta}(x)$ is the MLE of $\theta$ given one measurement $x$. Denote $\bar{\bm{u}} = (\bar{\tau},\bar{a}) = \slhc{\bm{x}}$, $\hat{\bm{u}} = (\hat{\tau},\hat{a})$ as the MPPLE, then:
\begin{enumerate}
\item If $\hat{\theta}$ is strictly increasing and $g$ is strictly decreasing, then $\bar{\tau}=\hat{\tau}$.
\item If, in addition, $\hat{\theta}(x) = x$ holds for all $x$, then $\bar{\bm{u}} = \hat{\bm{u}}$.
\end{enumerate}
\end{theorem}
This theorem provides evidence that, subject to the above monotonicity requirements of the measurement model, the full maximum likelihood estimator of $\slhc{\bm{\theta}}$ should display improved performance over that of $\slhc{\bm{x}}$, which is now characterized as the partial profile likelihood estimator obtained by discarding the measurements that do not correspond to edges on the MST of $\bm{x}$ (save, of course, for order information which was used to obtain the MST). We now prove the theorem.
\begin{proof} From \eqref{eq:partial_likelihood_in_proof}, \eqref{eq:definition_of_partial_likelihood} and since $u_e = \theta_e$ for $e\in T$, the log-likelihood is
\begin{equation}\label{eq:log_partial_likeli_hood}
	\log\,\mathcal{L}_{pp}(\bm{u};\bm{x}) = \underset{T\in \MST{\bm{u}}}{\text{max}}  \sum _{e\in T} \log G_{u_e}(x_e).
\end{equation}
It follows from Proposition~\ref{prop:MLE_u_profile_likelihood} that $\hat{\bm{u}} = \alpha(\widehat{T^{\bm{u}}})$, where $\widehat{T^{\bm{u}}}$ is the MLE of weighted spanning tree from 
\begin{equation} \label{eq:mle_of_the_weighted_spanning_tree}
	\widehat{T^{\bm{u}}} = \arg\underset{T^{\bm{u}} \in \wspt{O}}{\max} \sum _{e\in T} \log G_{u_e}(x_e).
\end{equation}
Since the variables are independent, summation in~\eqref{eq:mle_of_the_weighted_spanning_tree} is maximized if and only if each of the summands is maximized. Thus we get the MLE weight  of a given spanning tree $T$: $\hat{u}_e = \hat{\theta}_e$ for all $e\in T$. Then the maximum likelihood MST $\widehat{T}$ is
\begin{equation} \label{eq:mpple_mst_t}
	\widehat{T}  =  \arg \underset{T\in \spt{O}}{\max}  \sum _{e\in T}\log g(x_e).
\end{equation}

Since the function $x\mapsto\log g(x)$ is a strictly decreasing function of $x$, Corollary~\ref{cor:mst invariance under increasing functions} implies that $T$ is an MST of $\bm{x}$ if and only if $T$ is also a {\it maximum} spanning tree of the weight $\log g(\bm{x})$, implying that the total weight, $ \sum _{e\in T}\log g(x_e)$, along this tree is the maximum possible over all spanning trees. So far we have shown that the estimated MST coincides with an MST of the measurement, that is: $\widehat{T}\in \MST{\bm{x}}$.

In order for the dendrogram structures obtained from these estimated (weighted) trees through the map $\alpha$ to coincide, it suffices that the estimated weights along these trees be ordered in the same way. Indeed, this is true because $\hat\theta$ is a strictly increasing function.

Finally, to prove the second assertion, we observe that for $e\in \widehat{T}$ we have $\hat{u}_e=\hat{\theta}_e=x_e$. Since $\widehat{T}$ is an MST of $\bm{x}$, this implies that $\hat{\bm{u}} = \alpha(\widehat{T^{\bm{u}}})$ coincides with $\bm{u}=\slhc{\bm{x}}$.
\end{proof}

\subsubsection{Exponential Families}\label{section:exponential family case} 
We now obtain a sufficient condition for Theorem~\ref{thm:slhc_maxi_PLE} to apply in the case that the noise model, $G_\theta$ is a single parameter exponential family~\cite{letac1992lectures} ; that is,
we assume 
\begin{equation}\label{eq:exponential_family_general_form}
G_{\theta}(X) = p(X|\theta) = h(X)\text{exp}\left \{ \left(C(\theta)T(X) - A(\theta)\right)\right \}.
\end{equation}
We assume too, as is customary, that $C(\theta)$ is a $1-1$ function, so that the exponential family \eqref{eq:exponential_family_general_form} can be transformed to canonical form:
\begin{equation}
	p(X|\eta) = h(X) \text{exp}\left \{ \left(\eta T(X) - B(\eta) \right)\right \}
\end{equation}
where $\eta$ and $\theta$ are related by $\eta=C(\theta)$, and $A(\theta)=B(C(\theta))$ and, consequently, $ \tfrac{\text{d}B}{\text{d}\eta} =\frac{\text{d}A}{\text{d}\theta} \left( \frac{\text{d}C}{\text{d}\theta} \right)^{-1}$.

In the continuously differentiable case the MLE $\hat{\eta}(x)$ of $\eta$ given a measurement $x$ is a solution of the equation $\frac{\text{d}B}{\text{d}\eta} = T(x)$~(see e.g.~\cite{letac1992lectures}). The $1-1$ property yields that the MLE for $\eta$ satisfies $\hat \eta(x)=C(\hat \theta(x))$, where $\hat \theta$ is the MLE for $\theta$, and satisfies 
\begin{equation}\label{eq:mle_exponential_family_theta}
	\frac{\text{d}A}{\text{d}\theta} \left( \frac{\text{d}C}{\text{d}\theta} \right)^{-1} \bigg|_{\hat{\theta}}= T(x).
\end{equation}

For $g(x) = \text{log}p(x|\hat{\theta}) = \text{log}h(x) + C(\hat{\theta})T(x) - A(\hat{\theta})$, we obtain
\begin{equation} \label{eq:conditon_for_exponential_negtive}
		\frac{\text{d} g}{\text{d}x} =  \frac{\text{d} \text{ log}h}{\text{d}x}  + C(\hat{\theta})\cdot \frac{\text{d}T}{\text{d}x}
\end{equation}
using $\frac{\text{d}C}{\text{d}\theta}\big|_{\hat{\theta}} \cdot T(x) = \frac{\text{d}A}{\text{d}\theta}\big|_{\hat{\theta}}$.

According to the hypotheses, $\hat\theta(x)$ is strictly increasing and $\frac{\text{d} g}{\text{d}x}<0$, so that the structure of $\slhc{\bm{x}}$ is the MPPLE of the dendrogram structure. If, in addition,   $\hat{\theta}(x) = x$ holds for all $x$, then $\slhc{\bm{x}}$ coincides with the MPPLE of $\bm{u}$.

\subsubsection{Example based on the log-normal distribution}
\label{sec:An_example_based_on_log-normal_distribution}
We assume that each $X$ follows a log-normal distribution $\text{ln}\mathcal{N}(\mu,\sigma)$ with identical variance parameter $\sigma$.  
We model the relationship between the measurement and the true metric $\bm{\theta}$ as follows: we require $\mu_e = \text{ln}\theta_e$ and $\sigma_e = \sigma$ for all $e\in{O\choose 2}$, 
\begin{equation}\label{eq:data_model_probability}
	G_{\theta_e}(X_e)=
		\frac{1}{\sigma X_e \sqrt{2\pi}}
		\text{exp}\left(
			-\frac{(\text{ln}X_e - \text{ln}\theta_e)^2}{2\sigma ^2}
		\right).
\end{equation}

This is an exponential family with $h(X_e) = \frac{e^{-\frac{\text{ln}^2 X_e}{2\sigma^2}}}{\sqrt{2\pi}\sigma X_e}$, $T(X_e) =  \text{ln}X_e$, $C(\theta_e) = \frac{\text{ln}\theta_e}{\sigma^2}$ and $A(\theta_e) = \frac{\text{ln}^2\theta_e}{2\sigma^2}$. In this context, 
\begin{equation}
	\frac{\text{d}A(\theta_e)}{\text{d}\theta_e} \left( \frac{\text{d}C(\theta_e)}{\text{d}\theta_e} \right)^{-1}  = 
	\frac{\text{ln}\theta_e}{\sigma^2\theta_e} \left(\frac{1}{\sigma^2\theta_e}\right )^{-1} = \text{ln}\theta_e .
\end{equation}
Given a measurement $x_e$, we have $\hat{\theta}_e(x_e) = x_e$ which is strictly increasing. It follows that
\begin{equation}
	\frac{\text{d} \text{ log}h(x_e)}{\text{d}x_e}  + C(\hat{\theta}_e(x_e))\frac{\text{d}T(x_e)}{\text{d}x_e} = -\frac{\text{ln}x_e}{\sigma^2 x_e} 
	-\frac{1}{x_e} + \frac{\text{ln}x_e}{\sigma^2}\frac{1}{x_e} = -\frac{1}{x_e} < 0 .
\end{equation}
For this model, then, the MPPLE is just the SLHC. Note that both conditions of Theorem~\ref{thm:slhc_maxi_PLE} are satisfied in this example irrespective of the value of $\sigma$. In other words, estimating the parameter $\sigma$ is redundant for finding the MPPLE.

To illustrate this, we provide a simulation: 
\begin{itemize}
	\item[] \textbf{Step 1}. Samples of weights for 5 points are generated from the uniform distribution over the space $(0,100)^{5\choose{2}}$ until one is found that satisfies the triangle inequality. 
	\item[] \textbf{Step 2}. For a specific value of $\sigma$, a sample of $\bm{x}$ is generated according to \eqref{eq:data_model_probability}.
	\item[] \textbf{Step 3}. The values of $\bar{\bm{u}}=\slhc{\bm{x}}$, the MPPLE $\hat{\bm{u}}$, and the true value of $\bm{u} = \slhc{\bm{\theta}}$ are calculated. 
\end{itemize}

To quantify the error in this simulation, we use the $\ell_1$ distance (see~\cite{boorman1973metrics})
\begin{equation}
d(\bm{u}',\bm{u}) = \sum_{e\in{O\choose 2}}|u_e' - u_e|.
\end{equation}

The standard deviation $\sigma$ is decremented from $e^{0}$ to $e^{-8}$ in steps of $e^{-0.2}$. At each value of $\sigma$ , Steps 1-3 are repeated 10000 times, and incorrect  dendrogram structures and mean errors are calculated. Plots of the proportions of incorrect dendrogram structures and mean errors are given in Figures~\ref{fig:CorrectRatio} and~\ref{fig:RMSERROR} respectively.  As is readily seen, the results for SLHC and MPPLE are identical. 

Ideally in the simulation we would like  to sample
              general metrics from metric cones consistent with our
              model description  and theoretically analysis. However,
              randomly selecting a large number of samples from a metric cone is
              computationally unrealistic because of the problem of checking triangle
              inequalities to determine whether a given random weight
              is in fact a metric. It would be much easier, for
              instance, to select samples from a specific metric such as
              the Euclidean metric or the $\ell_\infty$ metric.    We have chosen, in this paper, to
              perform our simulations with a small number of data points (5
              actually) because  1),  the simulation is only an
              illustration of our rigorously proved  theoretical
              result,  2) calculation for a significantly larger
              number of points would be be computationally intensive
              as we have stated above; 3) a small number of data
              points enables easier comprehension of the results.

\begin{figure}%
\begin{center}
\parbox{0.45\columnwidth}{%
\includegraphics[width=0.4\columnwidth]{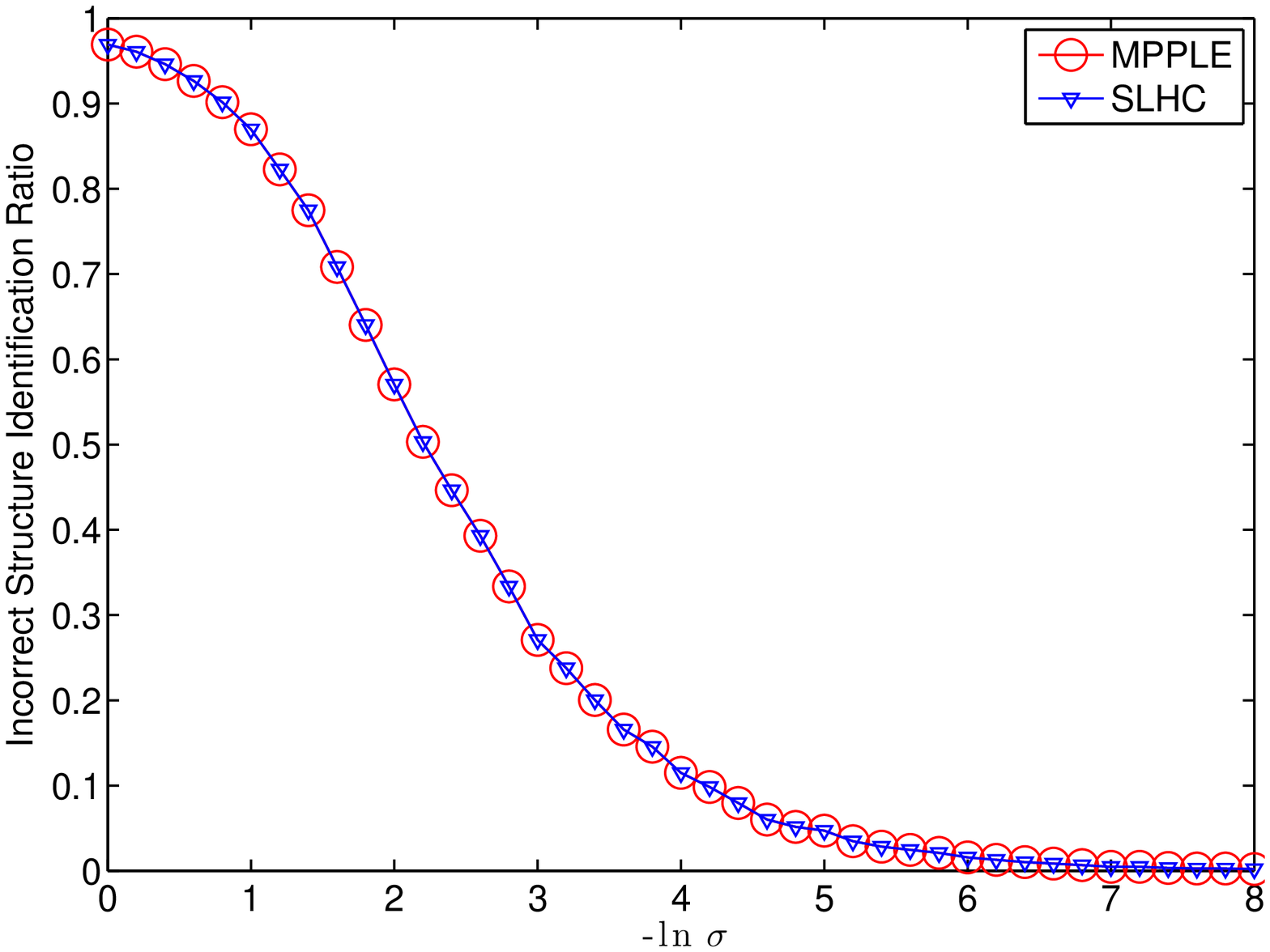}
\caption{Incorrect structure identification ratio versus $\sigma$.}%
\label{fig:CorrectRatio}}%
\qquad
\parbox{0.45\columnwidth}{%
\includegraphics[width=0.4\columnwidth]{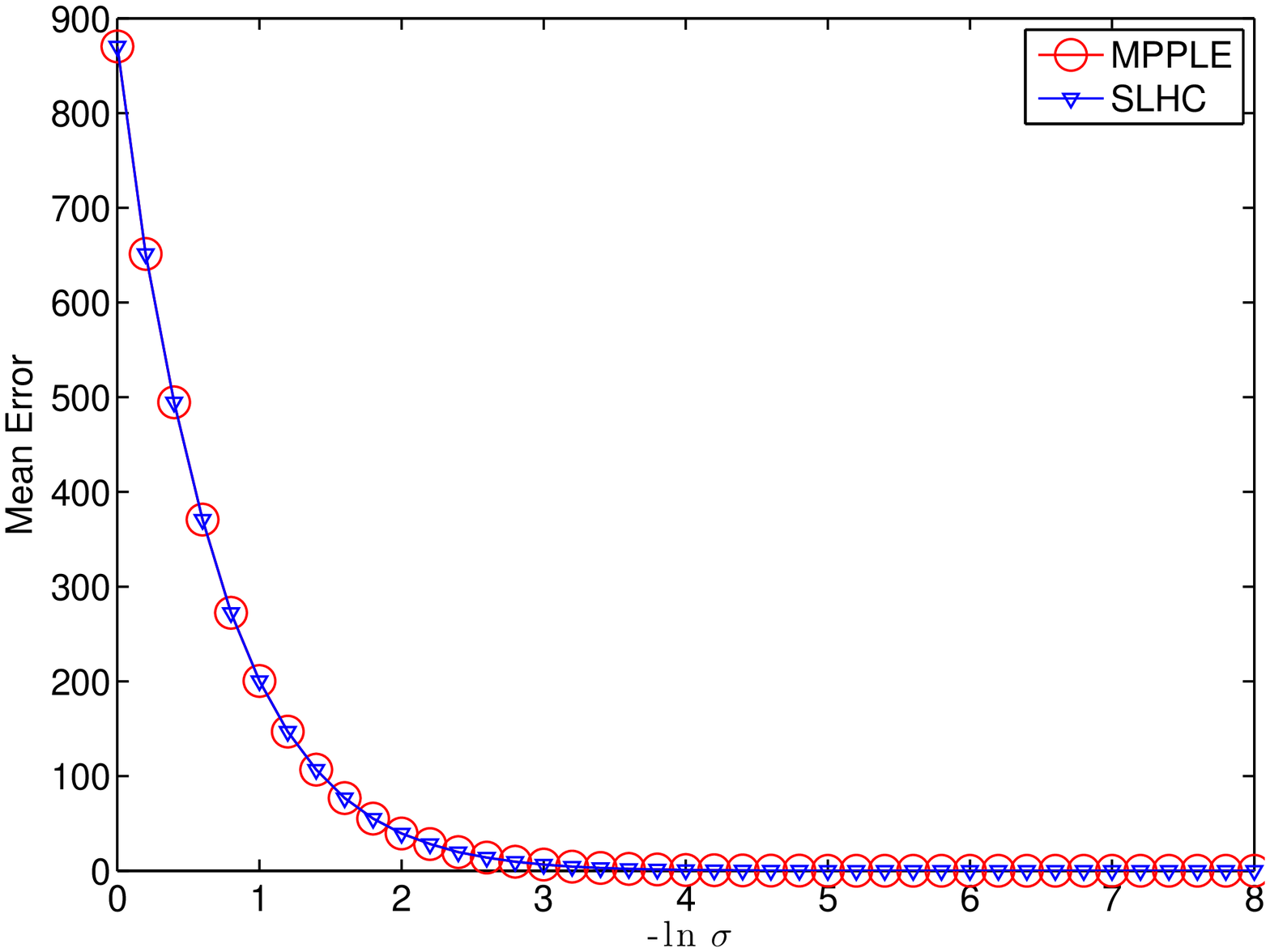}
\caption{Mean error of the ultra-metric versus $\sigma$.}%
\centering
\label{fig:RMSERROR}}%
\end{center}
\end{figure}%

\section{Consistency of SLHC}\label{section:consistency}
In this section we demonstrate the consistency of SLHC {\it as an estimator}, that is: we consider the consistency of 
\begin{equation}
	\bar{\bm{u}}(\bm{x}):=\slhc{\bm{x}}
\end{equation}
as an estimator of $\slhc{\bm{\theta}}$. Prior to the discussion, we require a technical lemma:
\begin{lemma}\label{lemma:continuity of alpha}
The map $\alpha:\wspt{O}\to\ultra{O}$ defined in Equation~\eqref{eqn:ultra-metric from a tree} is continuous with respect to the standard topology on $\ultra{O}$ induced from $\RR^{O\choose 2}$.
\end{lemma}
\begin{proof} First, since the connected components of $\wspt{O}$ are open, it suffices to verify the continuity of $\alpha$ when restricted to the component $K_T$ corresponding to an arbitrarily chosen tree $T\in\spt{O}$.

For a fixed $T$, $K_T$ is homeomorphic to the subspace of non-negative vectors in $\RR^{n-1}$, written as $(w_e)_{e\in T}$, with the standard topology. For each fixed $xy\in\binom{O}{2}$, the $xy$-coordinate of $\alpha(T^w)$, equal to $\max_{e\in p(T)_{xy}}w_e$, is the maximum of a finite collection of continuous real-valued functions of the vector $(w_e)_{e\in T}$. It follows that $\alpha$ is continuous over $K_T$, as desired.
\end{proof} 

Consistency is a characteristic of the limiting performance of an estimator as the number of samples tends to infinity. Since SLHC does not process multiple samples at a time, a natural way to take account of information contained in a sequence $\left(\bm{x}(1),\ldots,\bm{x}(N),\ldots\right)$ of measurements of $\bm{\theta}$ is as follows: for each edge $e\in{O\choose 2}$ obtain the MLE estimate $\bar\theta_e(N)$ of $\theta_e$ from the set $\{x_e(1),\cdots,x_e(N)\}$ to form the weight $\bm{\bar\theta}_N$, and define
\begin{equation}
	\bar{\bm{u}}_N\big(\bm{x}(1),\ldots,\bm{x}(N)\big):=\slhc{\bm{\bar\theta}_N}
\end{equation}

\begin{theorem}\label{thm:consistency_of_SLHC}
SLHC is a consistent estimator, that is: $\bar{\bm{u}}_N$ converges in probability to $\bm{u}$.
\end{theorem}

\begin{proof} Since we are interested in convergence in probability, we may omit from consideration the set of metrics $\bm{\theta}$ that are not generic. In particular, $\bm{\theta}$ has exactly one MST, denoted $T_\infty$, by Lemma~\ref{lemma:spanning tree decomposition}. Moreover, we have $\bm{\bar\theta}_N$ converging in probability to $\bm{\theta}$ because MLE is consistent, resulting in all but finitely many of the $\bm{\bar\theta}_N$ being generic. Then, without loss of generality, each $\bm{\bar\theta}_N$ has exactly one (weighted) MST, denoted $T_N$.  
We conclude that $(T_N,\bm{\bar\theta}_N)$ converges to $(T_\infty,\bm{\theta})$ in probability. Since $\alpha$ is continuous (Lemma~\ref{lemma:continuity of alpha}), we may apply the continuous mapping theorem~\cite{billingsley2013convergence} to obtain $\bar{\bm{u}}_N = \alpha(T_N,\bm{\bar\theta}_N) \overset{p}{\rightarrow} \alpha(T_\infty,\bm{\theta})$; since Proposition~\ref{prop:slhc and cones}(a) gives $\alpha(T_\infty,\bm{\theta}) = \slhc{\bm{\theta}} = \bm{u}$, we are done.
\end{proof}

\textbf{A sample simulation}: Simulations have been performed to illustrate the consistency property with the distribution model discussed in Section~\ref{sec:An_example_based_on_log-normal_distribution}. 
In the simulation the steps were as follows:
\begin{itemize}
\item[] \textbf{Step 1}. As Step 1 in Section~\ref{sec:An_example_based_on_log-normal_distribution};
\item[] \textbf{Step 2}. $N$ samples $\{ \bm{x}_1,\cdots,\bm{x}_N\}$ were randomly generated according to \eqref{eq:data_model_probability} for a specific $\sigma$, and $\bar{\bm{\theta}}_N$ was calculated;
\item[] \textbf{Step 3}. The ultra-metric $\bar{\bm{u}}_N = \slhc{\bar{\bm{\theta}}_N}$ was calculated and compared to the true value $\bm{u} = \slhc{\bm{\theta}}$.
\end{itemize}

In these simulations, $N$ was incremented from  $1$ to $2^{16}$ in powers of $2$, and  Steps 1-3 were repeated 10000 times. Again incorrect dendrogram structure identifications and mean errors were calculated. For four values of $\sigma = 0.3,0.2,0.1,0.05$, curves were plotted of proportions of incorrect identifications and mean errors in Figures~\ref{fig:Consistency_ratio} and~\ref{fig:Consistency_error} respectively. In both figures we observe convergence to the true value as $N\to \infty$. 

\begin{figure}%
\begin{center}
\parbox{0.45\columnwidth}{%
\includegraphics[width=0.4\columnwidth]{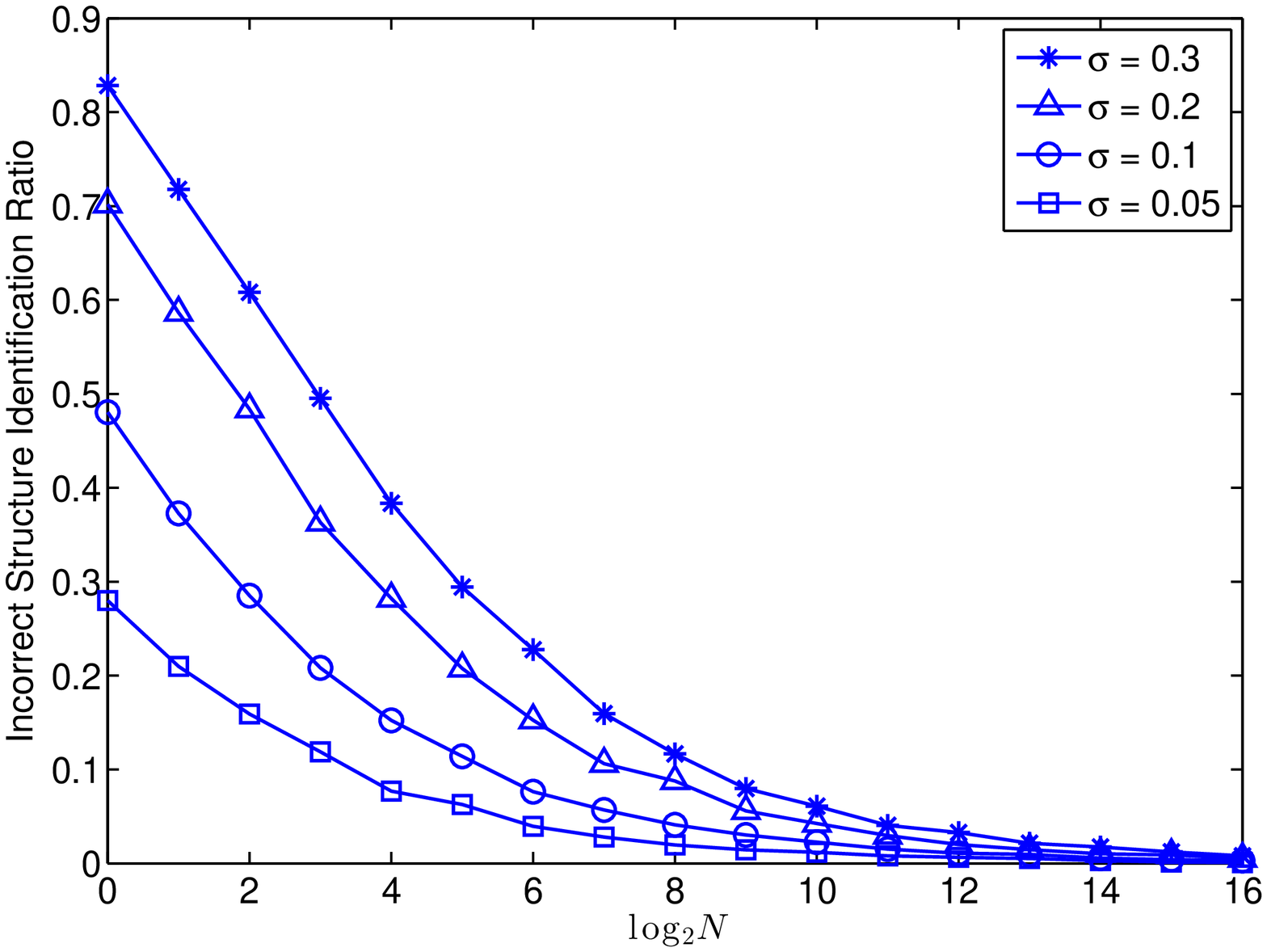}
\caption{Incorrect structure identification ratio versus $N$.}%
\label{fig:Consistency_ratio}}%
\qquad
\parbox{0.45\columnwidth}{%
\includegraphics[width=0.4\columnwidth]{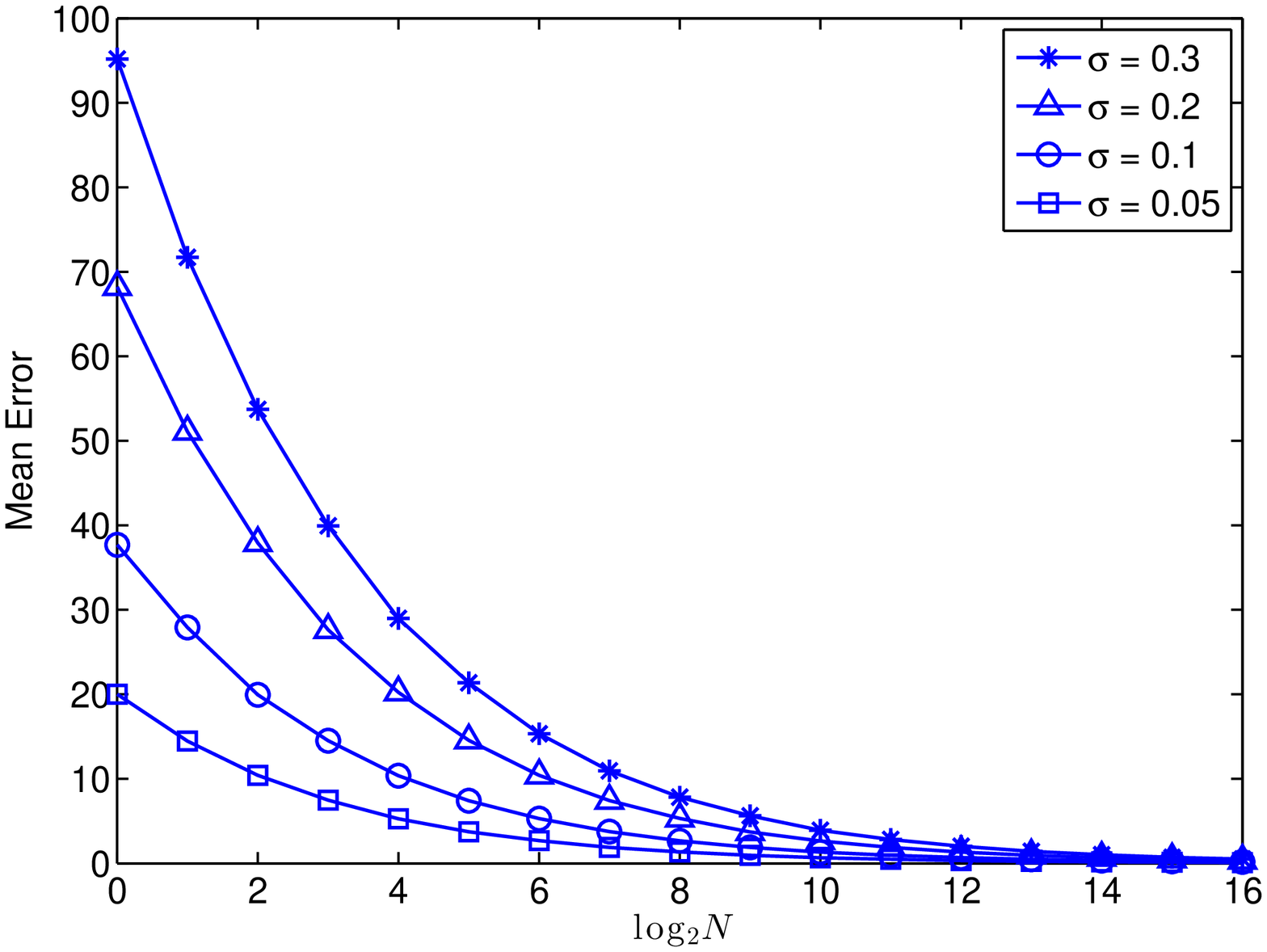}
\caption{Mean error of the ultra-metric versus $N$.}%
\centering
\label{fig:Consistency_error}}%
\end{center}
\end{figure}%

\section{Conclusion}\label{section:conclusion}

Our work here is an effort to address the issue of uncertainty of data surrounding the problem of hierarchical clustering. We restrict attention to the natural case where differences between data are given in terms of a metric but the information about that metric contains uncertainty modelled by a parametrized  (by the true metric) family of distributions. The hierarchical clustering approach is described in terms of an ultra-metric associated with the metric. The problem then becomes on of statistical estimation where the  parameter to be estimated is an ultra-metric. 

Despite its evident shortcomings, SLHC is proposed as the most natural algorithm for clustering in this context, following the work of Carlsson and Memoli, though with a simplification of their treatment. We show also that SLHC is equivalent to maximum partial profile likelihood estimator under reasonable models (see Theorem~\ref{thm:slhc_maxi_PLE}) for the uncertainty. SLHC is shown, also, to be a consistent estimator for multiple measurements.  The proofs of these results have required a study of the structure of the cone of weights  giving rise to a specific minimum spanning tree.

Our work strongly suggests that the use of the full maximum likelihood estimator, if available and computable, will outperform SLHC in contexts where the uncertainty is high. One focus of our future work will be on achieving computationally feasible but good approximations to the MLE.  

\section*{Acknowledgements}
This work was supported by the National Science Fund of China under Grant 61025006 and 61422114; the US Air Force Office of Science Research under Grant MURI FA9550-10-1-0567 and FA9550-12-1-0418.

\bibliographystyle{gSCS}
\bibliography{SLHC}
\medskip
\appendices
\section{Proofs of geometric and combinatorial results}
  \label{app:Proofs_of_materials_in_preliminary}
\subsection{Proof of Proposition~\ref{prop:slhc characterization}}
\label{app:Proof of Proposition slhc characterization}
\begin{proof} Equation \eqref{eq:slhc as sup} guarantees that $\slhc{\cdot}$ satisfies the above requirements, and it remains to verify the uniqueness claim. Suppose $E,F:\metr{X}\to\metr{X}$ are as stated. Then (a) $\ultra{X}$ coincides with the fixpoint sets of $E$ and $F$, and for every $d\in\metr{X}$ one has $F(Ed)=Ed$ and $E(Fd)=Fd$ since $Ed,Fd\in\ultra{X}$. Applying (b) and (c), we obtain:
	\begin{equation}
	d\geq Ed\THEN Fd\geq F(Ed)=Ed.
	\end{equation}
	By symmetry, $Ed\geq Fd$ for all $d\in\metr{O}$ as well.
\end{proof}

\subsection{Combinatorics of Minimum Spanning Trees}
\subsubsection{Proof of Lemma~\ref{lemma:mst comparison}}
\label{app:proof_of_lemma_mst_comparison}
In the rest of this section we will make extensive use of Kruskal's metric on $\spt{O}$:
\begin{equation}\label{eq:Kruskal metric}
	\kappa(T,T')=\card{T\minus T'}=\card{T'\minus T}.
\end{equation}
Recall that $\spt{O}$ forms a connected graph under the adjacency defined by $T\sim T'$ if and only if $\kappa(T,T')=1$. In fact, the set $\MST{d}$ of minimum spanning trees of $d$ is a connected sub-graph for any weight $d$. This observation underlies Kruskal's algorithm for computing MSTs \cite{West-textbook}.

\begin{proof} Suppose there are trees in $\MST{v}$ which are not MSTs of $w$. Then we may pick $T\in\MST{v}$ and $T^\ast\in\MST{w}$ so that $\card{T^\ast\minus T}$ is as small as possible, yet positive. We proceed to derive a contradiction by applying Kruskal moves as follows.
 	
First, find an edge $e^\ast\in T^\ast$ with $e^\ast\notin T$. Then $T\cup\{e^\ast\}$ contains a cycle of edges, $\ell$. We must have $v_{e^\ast}>v_f$ for all edges $f\in\ell,f\neq e^\ast$: if not, replace any violating edge of $\ell$ with the edge $e^\ast$ to obtain a tree $T^\sharp$ with either $\total{T^\sharp}{v}<\total{T}{v}$ or with $T^\sharp\in\MST{v}$ and $\card{T^\ast\minus T^\sharp}<\card{T^\ast\minus T}$~ ---~ a contradiction in either case. 
 	
Thus, by the assumption about $v$ and $w$, we have $w_{e^\ast}>w_f$ for all $f\in\ell$. Since $T^\ast$ cannot contain $\ell$, we find an edge $e\in\ell\minus T^\ast\subseteq T\minus T^\ast$. Since $w_e<w_{e^\ast}$, the tree obtained from $T^\ast$ by replacing $e^\ast$ with $e$ is a spanning tree of lower total weight under $w$ than $T^\ast$~ ---~ a contradiction again.
\end{proof}
 
\subsubsection{Proof of Corollary~\ref{cor:mst invariance under increasing functions}}
\begin{proof} Setting $v=g\circ w$ we observe that $v_e<v_f\IFF w_e<w_f$. Lemma~\ref{lemma:mst comparison} then provides us with the desired equality $\MST{w}=\MST{g\circ w}$.
\end{proof}

\subsection{Cones and the Geometry of SLHC}
\subsubsection{Proof of Lemma~\ref{lemma:cone of a tree}}
\begin{proof} Lemma~\ref{lemma:cone of a tree} is a corollary of Proposition~\ref{prop:slhc and cones}.
It suffices to take $a,b\in C(T)$, $\lambda\geq 0$ and to show that $a+\lambda b\in C(T)$ as well. Take any $xy\in{O\choose 2}$ and set $P=p(T)_{xy}$. We have:
\begin{equation}
 	\alpha(T^{a+\lambda b})_{xy}=
 	\max_{e\in P}(a_e+\lambda b_e)\leq
 	\max_{e\in P}a_e+\lambda\cdot\max_{e\in P}b_e\leq
 	a_{xy}+\lambda b_{xy}=
 	(a+\lambda b)_{xy}
\end{equation}
as desired. According to Proposition~\ref{prop:slhc and cones}, $a+\lambda b\in C(T)$. $C(T)$ is closed, because it is defined by a finite number of closed conditions.
\end{proof}
 
\subsubsection{Proof of Lemma~\ref{lemma:sl inverse images}}
\begin{proof} The proof follows directly from part (c) of Proposition~\ref{prop:slhc and cones} and Lemma~\ref{lemma:spanning tree decomposition}. Lemma~\ref{lemma:sl inverse images} is also a corollary of Proposition~\ref{prop:slhc and cones}.
\end{proof}
  
\subsubsection{Point Pre-Images under SLHC}
The preceding proofs depend on the following characterization of point pre-images of the map $\slhc{\cdot}$:
\begin{proposition}\label{prop:slhc and cones}
For all $T\in\spt{O}$ and $w\in\weight{O}$ one has:
\begin{equation}\label{eq:MST condition}
	T\in\MST{w}\IFF w\in C(T)\IFF w\geq\alpha(T^w).
\end{equation}
Moreover,
\begin{enumerate}
	\item[(a)] $T\in\MST{w}$ if and only if $\slhc{w}=\alpha(T^w)$;
	\item[(b)] If $T\in\MST{w}$ then $T\in\MST{\slhc{w}}$;
	\item[(c)] Finally, for any $u\in\ultra{O}$ one has $\slhc{w}=u$ if and only if $w$ and $u$ coincide on a common MST.
\end{enumerate}
\end{proposition}
\begin{proof} We first prove \eqref{eq:MST condition}. Denote $v=\alpha(T^w)$ for short. Suppose $T\in\MST{w}$ and fix $xy\in{O\choose 2}$. If $w_{xy}<v_{xy}$, then $p(T)_{xy}\cup\{xy\}$ is a cycle containing an edge $e\neq xy$ with $w_e=v_{xy}>w_{xy}$. But then the tree $Q=(T-e)\cup\{xy\}$ is a spanning tree of $K_O$ with $\total{Q}{w}<\total{T}{w}$ -- a contradiction. Conversely, suppose $w\geq v$ and choose $Q\in\MST{w}$ with $\kappa(T,Q)$ -- the Kruskal distance from $Q$ to $T$ (Section~\ref{section:weighted trees}) -- as small as possible. If $\kappa(T,Q)>0$, then there is an edge $xy\in Q\minus T$. Find an edge $e\in p(T)_{xy}$ joining the two connected components of the forest $Q-xy$. Then $Q'=(Q-xy)\cup\{e\}$ is a spanning tree with $\kappa(T,Q')=\kappa(T,Q)-1$. Now, $w_{xy}\geq\alpha(T^w)_{xy}\geq w_e$ implies $Q'$ is an MST -- contradiction to the choice of $Q$. We conclude that $\kappa(T,Q)=0$, or, equivalently, $T\in\MST{w}$, as claimed.

Next let us prove properties (a)-(c).

For (a), let $u=\alpha(T^w)$. If $\slhc{w}=u$, then from \eqref{eq:slhc as sup} we have $u\leq w$ and hence $T\in\MST{w}$, by \eqref{eq:MST condition} we have just proved. Conversely, suppose $u\leq w$. Then $u\leq\slhc{w}$ by \eqref{eq:slhc as sup}, while coinciding with $w$ along $T$. For any $xy\in{O\choose 2}$ we then have $\slhc{w}_{xy}\leq \text{ max }(w_e)$ for all $e\in p(T)_{xy}$ through repeatedly applying the ultra-metric inequality to $\slhc{w}$; but this implies $\slhc{w}_{xy}\leq u_{xy}$ and we have $u=\slhc{w}$.

To establish (b), suppose $T\in\MST{w}$ and use (a) to rewrite this as $\slhc{w}=u$. In particular, for every $e\in T$ we have $\slhc{w}_e=w_e$, and the weighted trees $T^w$ and $T^u$ coincide. We conclude:
\begin{displaymath}
	\slhc{u}=\slhc{\slhc{w}}=\slhc{w}=u=\alpha(T^w)=\alpha(T^u).
\end{displaymath}
Applying (a) to the metric $u$, implies $T\in\MST{u}$, as desired.
      			
Finally, for (c), writing $\slhc{w}=u$ we use the fact that $T^w=T^u$ for {\it any} $T\in\MST{w}$. Conversely, suppose $T\in\MST{w}\cap\MST{u}$ and $T^w=T^u$. We have: $T\in\MST{u}$ and $u\in\ultra{O}$ imply $u=\slhc{u}=\alpha(T^u)$; $T\in\MST{w}$ implies $\slhc{w}=\alpha(T^w)=\alpha(T^u)=u$, as required.
\end{proof}

\section{Proof of Proposition~\ref{prop:partial_likelihood}}
\label{append:proof_of_independence_partial}

\begin{proof} Obviously we have
	\begin{equation}
	p(\bm{\theta}|T,\bm{u}) = p(\bm{\theta}^{\text{on}}|T,{\bm{u}}) p(\bm{\theta}^{\text{off}}|\bm{\theta}^{\text{on}},T,{\bm{u}}) 
	\end{equation}
	and $p(\bm{\theta}^{\text{on}}|T,{\bm{u}}) = \prod_{e\in T} \delta (\theta_e- u_e)$. Given $T$ and $\bm{u}$, $\bm{\theta}^{\text{on}}$ is known as well, therefore, $p(\bm{\theta}^{\text{off}}|\bm{\theta}^{\text{on}},T,{\bm{u}}) = p(\bm{\theta}^{\text{off}}|T,{\bm{u}})$.  
	\begin{equation}
	p(\bm{\theta}|T,{\bm{u}}) = \prod_{e\in T} \delta (\theta_e- u_e) \cdot p(\bm{\theta}^{\text{off}}|T,{\bm{u}})
	\end{equation}
	
	\begin{equation}
	\begin{split}
	p(\bm{X}^{\text{on}}|T,{\bm{u}}) &= \int _{\bm{X}^{\text{off}}} p(\bm{X}|T,{\bm{u}})\dd \bm{X}^{\text{off}}\\
	&= \int _{\bm{X}^{\text{off}}} \int _{C(T,{\bm{u}})}  \prod_{e\in {O\choose{2}}} G_{\theta_{e}}(X_{e}) p(\bm{\theta}|T,{\bm{u}})\dd\bm{\theta} \dd\bm{X}^{\text{off}}\\
	&= \int_{C(T,{\bm{u}})}  \prod_{e\in T} G_{\theta_e}(X_e) p(\bm{\theta}|T,{\bm{u}}) \dd \bm{\theta}\\
	&= \prod_{e\in T} G_{u_e}(X_e)
	\end{split}
	\end{equation}
	\begin{equation}
	\begin{split}
	&p(\bm{X}|T,{\bm{u}}) = \int _{C(T,{\bm{u}})} \prod_{e\in{O\choose 2}} G_{\theta_{e}}(X_{e}) p(\bm{\theta}|T,{\bm{u}}) \dd \bm{\theta} \\
	&= \int _{\bm{\theta}^{\text{off}}} \int_{\bm{\theta}^{\text{on}}} G_{\theta_{e}}(X_{e}) \cdot \left [\prod_{e\in T} \delta (\theta_e-u_e) \cdot p(\bm{\theta}^{\text{off}}|T,{\bm{u}})\right] \dd\bm{\theta}^{\text{on}}\dd\bm{\theta}^{\text{off}}\\
	& =\prod_{e\in T} G_{u_e}(X_e) \int_{C(T,{\bm{u}})} \prod_{e \notin T}G_{\theta_{e}}(X_{e})p(\bm{\theta}^{\text{off}}|T,{\bm{u}})\dd \bm{\theta}^{\text{off}}
	\end{split}
	\end{equation}

	\begin{equation}\label{eq:independence_theta_0n_0ff}
	\begin{split}
	p(\bm{X}^{\text{off}}|\bm{X}^{\text{on}},T,{\bm{u}}) &= \frac{p(\bm{X}|T,{\bm{u}})}{p(\bm{X}^{\text{on}}|T,{\bm{u}})} = \int_{C(T,{\bm{u}})} \prod_{e \notin T}G_{\theta_{e}}(X_{e})p(\bm{\theta}^{\text{off}}|T,{\bm{u}})\dd\bm{\theta}^{\text{off}}\end{split}
	\end{equation}
	The right side of \eqref{eq:independence_theta_0n_0ff} is independent with $\bm{X}^{\text{on}}$, therefore,
	$p(\bm{X}^{\text{off}}|T,{\bm{u}}) = p(\bm{X}^{\text{off}}|\bm{X}^{\text{on}},T,{\bm{u}})$. Then \eqref{eq:two_parts_likelihood} follows. 
\end{proof}

\end{document}